\documentclass{article}

\usepackage{times}
\usepackage{graphicx} 
\usepackage{subfigure} 

\usepackage{natbib}
\usepackage{cite}
\usepackage{algorithm}
\usepackage{algorithmic}

\usepackage{hyperref}

\usepackage[accepted]{icmlstylefiles/icml2014}

\definecolor{WowColor}{rgb}{.75,0,.75}
\definecolor{SubtleColor}{rgb}{0,0,.50}

\newcounter{margincounter}

\usepackage{amssymb}
\usepackage{amsthm}
\usepackage{amsmath}
\usepackage{mathtools}

\newtheorem{theorem}{Theorem}[section]	

\newtheorem{lemma}[theorem]{Lemma}

\theoremstyle{definition} \newtheorem{definition}{Definition}

    \def\independenT#1#2{\mathrel{\setbox0\hbox{$#1#2$}%
    \copy0\kern-\wd0\mkern4mu\box0}} 
    
\icmltitlerunning{The Random Forest Kernel (and creating other kernels for big data from random partitions)}

\begin{document} 

\twocolumn[
\icmltitle{The Random Forest Kernel\\and creating other kernels for big data from random partitions}

\icmlauthor{Alex Davies}{ad564@cam.ac.uk}
\icmladdress{Dept of Engineering, University of Cambridge
			 Cambridge, UK}
\icmlauthor{Zoubin Ghahramani}{zoubin@cam.ac.uk}
\icmladdress{Dept of Engineering, University of Cambridge
			 Cambridge, UK}

\icmlkeywords{boring formatting information, machine learning, ICML}

\vskip 0.3in
]

\begin{abstract}

We present Random Partition Kernels, a new class of kernels derived by demonstrating a natural connection between random partitions of objects and kernels between those objects. We show how the construction can be used to create kernels from methods that would not normally be viewed as random partitions, such as Random Forest. To demonstrate the potential of this method, we propose two new kernels, the Random Forest Kernel and the Fast Cluster Kernel, and show that these kernels consistently outperform standard kernels on problems involving real-world datasets. Finally, we show how the form of these kernels lend themselves to a natural approximation that is appropriate for certain big data problems, allowing $O(N)$ inference in methods such as Gaussian Processes, Support Vector Machines and Kernel PCA.

\end{abstract}

\section{Introduction}

Feature engineering is often cited as the \emph{``the most important factor''} \cite{Domingos2012-go} for the success of learning algorithms. In a world of kernelized algorithms this means, unsurprisingly, that the kernel is the most important factor in the success of the underlying algorithm. Since there will likely never be a single best kernel for all tasks \cite{Wolpert1997-hh}, the best hope for kernelized modelling is to have many available standard kernels and methods to easily generate new kernels from knowledge and intuition about the problem domain.

The most commonly used kernels are generally analytically derived. An informal review of kernel machine papers shows that despite a large number of available kernels, for real valued inputs practitioners overwhelmingly use a very small subset: Linear, Periodic, Radial Basis Function (RBF), Polynomial, Spline and ANOVA, with RBF being far and away the most common, and almost exclusively used for high-dimensional data. These kernels are empirically effective for many tasks and there are instances where there are compelling arguments that they are accurate models of the data. However, it is unfortunate when they are used as ``default'' kernels, encoding at best a bias towards smoothness, which even then may not be the best assumption in real-world, high dimensional problems.

These also leave little recourse for constructing new kernels. They can be combined through the use of certain operators such as $+$ and $\times$ and some work has been done towards automatically searching these structures \cite{Duvenaud2013-rk}, but this is a time consuming and difficult task. An example of a general purpose method for constructing kernels are Fisher kernels \cite{Jaakkola1999-zp}. Fisher kernels provide a method to construct a sensible kernel from a generative model of the data. For cases where there is a natural generative model of the data, this is a compelling and intuitive kernel.

Clustering has a long history in machine learning and statistics precisely because it as a very simple and intuitive task. Originally, clustering methods were designed to find ``the best'' partition of a datset. More recent probabilistic methods allow uncertainty and instead learn a distribution on what the clusterings of the data might be. These distributions are known as partition distributions, and samples from them are random partitions.

In this paper we show how it is very simple to define a distribution over partitions, and how many existing algorithms trivially define such distributions. We show how a kernel can be constructed from samples from these distributions (random partitions) and that these kernels outperform common kernels on real-world data. We also show how the construction of the kernel leads to a simple approximation scheme that allows for scalable inference in SVMs, Kernel PCA and GP regression.

\section{Background}

\subsection{Kernel methods}
\label{sec:kernels}
A kernel is a positive semi-definite (PSD) function $k(a,b) : \mathcal{I} \times \mathcal{I} \rightarrow \mathbb{R}$ that is used to indicate the similarity of two points in a space $\mathcal{I}$. For a dataset $\mathbf{X} = \{\vec{x}_1,...,\vec{x}_N\}$, the \emph{Gram matrix} of $\mathbf{X}$ is $\mathbf{K} \in \mathbb{R}^{N\times N}$, where $\mathbf{K}_{ab} = k(\vec{x}_a,\vec{x}_b)$. Many machine learning problems can be formulated in terms of this Gram matrix, rather than the more traditional design matrix $\mathbf{X}$. These methods are collectively refered to as \emph{kernel machines} \cite{Herbrich2002-vd}. The choice of $k$ (and thus $\mathbf{K}$) is critical to the predictive performance of any of these algorithms, in the same way that the choice of a sensible feature space is to traditional methods. This is because the choice of the kernel implicitly sets the feature space being used: the kernel can always be viewed as $k(\vec{x}_a, \vec{x}_b) = \left<\phi(\vec{x}_a),\phi(\vec{x}_b)\right>$, where $\phi$ is the implicit feature space mapping.\\

\subsection{Iterative methods and pre-conditioning}
\label{sec:iterative}

One issue with kernel machines is that they can be very computationally expensive. To simply evaluate every element in a Gram matrix requires $O(N^2)$ operations, and many kernel machine algorithms are $O(N^3)$. Fortunately in some cases the solutions to these problems can be found using matrix-free iterative solvers. These solvers are called matrix-free because they do not require the full Gram matrix $\mathbf{K}$, only the ability to calculate $\mathbf{K}\vec{v}$ for arbitrary $\vec{v}$. When $\mathbf{K}$ can be stored in less that $O(N^2)$ space, and $\mathbf{K}\vec{v}$ calculated in less than $O(N^2)$ time, these methods can offer great computational savings over direct methods.

Matrix-free iterative methods include \emph{conjugate gradient} \cite{Shewchuk1994-sr}, for finding the minimum of a quadratic form and \emph{power iteration} \cite{Press2007-ie} for finding the eigenvectors of a matrix. Three common kernel machines that can use matrix free methods are Gaussian Processes; calculating the posterior mean requires a solution to $\mathbf{K}^{-1}\vec{y}$, which can be solved via conjugate gradient \cite{Shewchuk1994-sr}, SVMs, which can be solved using gradient descent on the objective, where the derivative only depends on $\mathbf{K}\alpha$ and Kernel PCA, which can be solved with power iteration.

The numerical stability and convergence rate of iterative methods are directly related to the condition number of the Gram matrix $\kappa(\mathbf{K})$. A low value of the condition number will lead to a numerically stable solution in a small number of iterations, whereas a problem with a very high condition number may not converge at all. Conjugate gradient, for example, has a convergence rate of $\sqrt{\kappa(\mathbf{K})}$.

If we can construct a matrix $\mathbf{B}$, for which we can efficiently evaluate $\mathbf{B}\vec{v}$, and where $\kappa(\mathbf{B}\mathbf{K}) \ll \kappa(\mathbf{K})$, we can perform these iterative methods on the transformed system $\mathbf{B}\mathbf{K}$. This matrix $\mathbf{B}$, known as a \emph{pre-conditioning} matrix, can be thought of as an ``approximate inverse'', that brings $\mathbf{K}$ closer to the identity matrix (which has the lowest possible condition number of 1 and would require only one iteration to solve).\\

\subsection{Random partitions}

A cluster $C$ in a dataset $\mathcal{D}$ is a non-empty subset of $\mathcal{D}$. A partition of $\mathcal{D}$ is the segmentation of $\mathcal{D}$ into one or more non-overlapping clusters $\varrho = \{C_1, ... , C_j\}$. ie:

\begin{eqnarray}
C_i \cap C_j &= \emptyset\\
\bigcup_i C_i &= \mathcal{D}
\end{eqnarray}

The following is an example dataset and an example partition of that dataset:

\begin{eqnarray}
\mathcal{D} =& \{a,b,c,d,e,f\}\\
\varrho =& \{\{a,e\},\{c\},\{d,b,f\}\}
\end{eqnarray}

A \emph{random partition} of $\mathcal{D}$ is a sample from a partition distribution $\mathcal{P}$. This distribution is a discrete pdf that represents how likely a given clustering is. We will use the notation $\varrho(a)$ to indicate the cluster that the partition $\varrho$ assigns to point $a$.\\

Random partitions have been studied extensively in the field of non-parametric Bayesian statistics. Well-known Bayesian models on random partitions are the Chinese Restaurant Process \cite{Aldous1985-lo}, the Pitman-Yor Process \cite{Pitman1997-iv} and the Mondrian Process \cite{Daniel_M_Roy2009-bo}, while many other combinatorial models can be used to define random partitions.

\subsection{Defining distributions with programs}

Ordinarily a distribution is defined by its probability density function, or alternatively by some other sufficient description such as the cumulative density function or moment generating function. However, a distribution can also be defined by a program that maps from a random seed $R$ to the sample space of the distribution. This program can be viewed in two ways: firstly, as a transformation that maps from a random variable to a new output space, and alternatively as a program that generates samples from the desired distribution.

Thus any program that takes a datset as input and outputs a random partition defines a partition distribution. This is the representation that will allow us to easily construct Random Partition Kernels.

\section{Random Partition Kernels}

A partition distribution naturally leads to a kernel in the following way:

\begin{definition}
Given a partition distribution $\mathcal{P}$, we define the kernel \[k_\mathcal{P}(a,b) = \mathbb{E}\left[I\left[\varrho(a) = \varrho(b)\right]\right]_{\varrho \sim \mathcal{P}}\] to be the random partition kernel induced by $\mathcal{P}$, where I is the indicator function.
\end{definition}

That is, the kernel is defined to be the fraction of the time that two points are assigned to the same cluster.

\begin{lemma}
$k_\mathcal{P}(a,b)$ constitutes a valid PSD kernel.
\end{lemma}

\begin{proof}

First we define:

\[k_\varrho(a,b) = I\left[\varrho(a) = \varrho(b)\right]\]

To prove that $k_\mathcal{P}$ is PSD, we decompose the expectation into the limit of a summation and show that the individual terms are PSD.

\begin{align}
k_\mathcal{P}(a,b) &= \mathbb{E}\left[I\left[\varrho(a) = \varrho(b)\right]\right]_{\varrho \sim \mathcal{P}}\\
&= \lim_{n \rightarrow \infty} \frac{1}{n}\sum_{\varrho \sim \mathcal{P}}^n I\left[\varrho(a) = \varrho(b)\right]\\
&= \lim_{n \rightarrow \infty} \frac{1}{n}\sum_{\varrho \sim \mathcal{P}}^n k_\varrho(a,b)
\end{align}

For any dataset of size $N$, the kernel matrix for $k_{\varrho}$ will be an $N \times N$ matrix that can be permuted into a block diagonal matrix of the following form:

\[\mathbf{Z}\mathbf{K}_{\varrho}\mathbf{Z}^\intercal = \begin{bmatrix}
\mathbf{1} & \mathbf{0} & \ldots & \mathbf{0}\\
\mathbf{0} & \mathbf{1} & \hdots & \mathbf{0}\\
\vdots & \vdots & \ddots & \mathbf{0}\\
\mathbf{0} & \mathbf{0} & \mathbf{0} & \mathbf{1}\\
\end{bmatrix}\]

where $\mathbf{1}$ is a matrix with all entries equal to 1, $\mathbf{0}$ is a matrix will all entries 0 and $\mathbf{Z}$ is a permutation matrix. Each $\mathbf{1}$ matrix represents a single cluster.\\

From this we can conclude that $\mathbf{Z}\mathbf{K}_{\varrho}\mathbf{Z}^\intercal$ is PSD, as it is a block matrix of PSD matrices. Further, since a permutation does not affect the eigenvalues of a matrix, $\mathbf{K}_{\varrho}$ is PSD. Since $\mathbf{K}_{\varrho}$ is PSD for any dataset, $k_{\varrho}$ must be a valid kernel. Finally, since a linear combination of kernels with positive coefficients is also a valid kernel, $k_{\mathcal{P}}$ is a valid kernel.

\end{proof}

\subsection{$m$-approximate Random Partition Kernel}

However in many instances, such as when defining a partition distribution with a program, it is not possible to analytically evaluate this quantity. Fortunately, its structure allows for a simple approximation scheme that only requires the ability to sample from the distribution $\mathcal{P}$. 

\begin{definition}
The $m$-approximate Random Partition Kernel is the fraction of times that $\mathcal{\varrho}$ assigns $a$ and $b$ to the same cluster over $m$ samples.
\[k_\mathcal{P}(a,b) \approx \tilde{k}_{\mathcal{P}}(a,b) = \frac{1}{m}\sum_{\varrho \sim \mathcal{P}}^m k_\varrho(a,b)\]
\label{defn:stoch_kernel}
\end{definition}

\begin{lemma}
If the samples from $\mathcal{P}$ are independent then the bound on the variance of the approximation to $k_\mathcal{P}(a,b)$ is $O\left(\frac{1}{m}\right)$.
\end{lemma}

\begin{proof}

If $\varrho$ are independent samples from $\mathcal{P}$, then 

\[k_\varrho(a,b) \sim Bernoulli\left(k_\mathcal{P}(a,b)\right)\]

and $\tilde{k}_{\mathcal{P}}(a,b)$ is the maximum likelihood estimator for $k_\mathcal{P}(a,b)$. The variance of the ML estimator for a Bernoulli is bounded by $\frac{1}{4m}$.
\end{proof}

\subsection{The Colonel's Cocktail Party}

This process of evaluating the kernel described in Definition \ref{defn:stoch_kernel} can be described succintly using a metaphor in the tradition of the CRP \cite{Aldous1985-lo} and IBP \cite{Griffiths2011-vj}. 

We consider The Colonel, who is the host of a cocktail party. He needs to determine the strength of the affinity between two of his guests, Alice and Bob. Neither Alice and Bob, nor the other guests, must suspect the Colonel's intentions, so he is only able to do so through surreptuous observation.

At the beginning of the evening, his $N$ guests naturally form into different groups to have conversations. As the evening progresses, these groups naturally evolve as people leave and join different conversations. At $m$ opportunities during the course of the evening, our host notes whether Alice and Bob are together in the same conversation.


As the Colonel farewells his guests, he has a very good idea of Alice and Bob's affinity for one another.

\subsection{Efficient evaluation}
\label{sec:efficient_eval}
As stated in Section \ref{sec:iterative}, a large class of kernel algorithms can be efficiently solved if we can efficiently solve $\mathbf{K}\vec{v}$ for arbitrary $\vec{v}$. A partition matrix $\mathbf{K}_\varrho$ can be stored in $N$ space and multiplied by a vector in $2N$ operations using the analytic form in Eq \ref{eq:partition_mul}. 

\begin{equation}
\left(\mathbf{K}_\varrho\vec{v}\right)_i = \sum_{j \in \varrho(i)} \vec{v}_j\label{eq:partition_mul}
\end{equation}

It follows that the kernel matrix for an $m$-approximate Random Partition Kernel requires $mN$ space and $2mN$ operations for a matrix-vector product. In an iterative algorithm this leads to an overall complexity of $2mNI$ operations.\\

For the Random Partition Kernel, we propose the following pre-conditioning matrix:

\begin{equation}
\mathbf{B} = m \sum_{\varrho\in\mathcal{P}}^m \left(\mathbf{K}_\varrho + \sigma\mathbf{I}\right)^{-1}
\end{equation}

Where $\sigma$ is set as a small constant to ensure the matrix is invertible. Due to the simple form of the individual cluster matrices, we can compute $\left(\mathbf{K}_\varrho + \sigma\mathbf{I}\right)^{-1}\vec{v}$ in only $2N$ operations and $\mathbf{B}\vec{v}$ in $2mN$ operations using the analytic form in Eq \ref{eq:partition_solve}.

\begin{equation}
\left(\left(\mathbf{K}_\varrho + \sigma\mathbf{I}\right)^{-1}\vec{v}\right)_i = \frac{1}{\sigma} \vec{v}_i - \frac{1}{|\varrho(i)| + \sigma}\sum_{j \in \varrho(i)} \vec{v}_j\label{eq:partition_solve}
\end{equation}

This small multiplicative overhead to the iterative solver greatly reduces the condition number and the number of iterations required for a solution in most cases.

\section{Constructing Kernels}

To demonstrate the potential of this method, we first show how to obtain random partitions from a large class of existing algorithms. Then we select two example partition distributions to propose two novel kernels for real-world data: the \emph{Random Forest Kernel} and the \emph{Fast Cluster Kernel}. We also show examples of how some common types of data can be directly interpreted as random partitions.

\subsection{Example random partitions}

To demonstrate the ease with which sensible random partitions can be constructed from exisiting machine learning algorithms, we provide a number of examples. This list is by no means exhaustive, but serves as an example of how random partitions can be defined.

\subsubsection{Stochastic clustering algorithms}

Any standard clustering algorithm (K-means \cite{MacKay2003-oh} and DBSCAN \cite{Ester1996-me} being just a couple of examples) generates a partition for a given dataset. Adding any element of randomness to the algorithms (if it does not already exist) results in a stochastic clustering algorithm.

Elements of the clustering algorithm that can be randomized include:

\begin{enumerate}
\item Initializations
\item Number of clusters
\item Subsets of the data (Bagging)
\item Subset of the features
\item Projections of the features
\end{enumerate} 

The output for a stochastic clustering algorithm is a random partition. As a simple concrete example, the output of K-means with random restarts returns a random partition. 

\subsubsection{Ensembled tree methods}

A large class of random partitions can be derived from \emph{ensembled tree methods}. We use this term to refer to any method whose output is a collection of random trees. By far the most well known of these are Random Forest \cite{Breiman2001-fl} and Boosted Decision Trees \cite{Freund1997-xn}, though it also includes Bayesian Additive Regression Trees \cite{Chipman2010-os}, Gradient Boosting Machines with decision trees \cite{Friedman2000-aj} and many others.\\

As a tree defines a hierarchical partition, it can be converted to a partition by randomly or deterministically choosing a height of the tree and taking the partition entailed by the tree at that height. A more detailed example of this is outlined in the Random Forest Kernel algorithm in Section \ref{sec:rf_algo}.

\subsection{Data-defined partitions}

There are a number of real-world examples where the data can be directly interpreted as samples from a random partition. 

\subsubsection{Collaborative filtering}

For example, in collaborative filtering for movie recommendations, we can cluster users by randomly sampling a movie and splitting them into two clusters: viewed and not viewed. Similarly, we can cluster movies but splitting them into clusters for which were viewed, or not, by a particular user.

\subsubsection{Text analysis}

Similarly, when considering text mining applications, we can view documents as binary partitions of words: those in the document and those that are not. 

\subsection{Example kernels}

To further understand the kernels and perform experiments, we select two example random partitions to construct the \emph{Random Forest Kernel} and the \emph{Fast Cluster Kernel}.

\subsubsection{Random Forest Kernel}

\label{sec:rf_algo}

\begin{algorithm}[tb]
   \caption{Random Forest Partition Sampling Algorithm}
   \label{alg:rf_sc}
\begin{algorithmic}
   \STATE {\bfseries Input:} $\mathbf{X} \in \mathcal{R}^{N \times D},\vec{y} \in \mathcal{R}^N,h$
   \STATE $T \sim {\rm RandomForestTree}(\mathbf{X},\vec{y})$
   \STATE $d \sim {\rm DiscreteUniform}(0,h)$
   \FOR{$a \in \mathcal{D}$}
   \STATE ${\rm leafnode} = T(a)$
   \STATE $\varrho(a) = \text{ancestor}(d, {\rm leafnode})$
   \ENDFOR
\end{algorithmic}
\end{algorithm}

To construct the Random Forest Kernel, we define the following random partition sampling scheme, summarized in Algorithm \ref{alg:rf_sc}.\\

Firstly a tree $T$ is generated from the Random Forest algorithm. A tree generated by Random Forest is a decision tree that is trained on a subset of the data and features; however, every datapoint in $\mathcal{D}$ is associated with a leaf node. For simplicity of explanation, we assume that $T$ is a binary tree with $2^h$ nodes, but this is not necessary in general. Next a height is sampled from $d \sim {\rm DiscreteUniform}(0,h)$ which determines at what level of the tree the partitioning is generated. The nodes of $T$ at height $d$ of the tree indicate the different clusters, and each datapoint is associated to the cluster whose node is the ancestor of the datapoint's node.\\

Using this partition sampling scheme in Definition \ref{defn:stoch_kernel} gives us the Random Forest Kernel. This kernel has a number of interesting properties worth higlighting:

(1) \emph{When used in Gaussian Processes, it is a prior over piece-wise constant functions}. This is true of all Random Partition Kernels. The posterior of a Gaussian Process on a toy 1-D dataset is shown in Figure \ref{fig:rf_posterior}.\\
(2) \emph{The GP prior is also non-stationary}. A property of real-world datasets that is often not reflected in analytically derived kernels is non-stationarity; where the kernel depends on more than just the distance between two points.\\
(3) \emph{It has no hyper-parameters}. This kernel has no hyper-parameters, meaning there is no costly hyper-parameter optimization stage. This normal extra  ``training'' phase in GPs and SVMs is replaced by the training of the random forests, which in most cases would be more computationally efficient for large datasets.\\
(4) \emph{It is a supervised kernel}. While generally kernels are unsupervised, there is nothing inherently incorrect about the use of a supervised kernel such as this, as long as it is only trained on the training data. While this could reasonably lead to overfitting for regression/classification, the experiments do not show any evidence for this. It can be used for supervised dimensionality reduction, a plot of which is shown in Figure \ref{fig:rf_pca}.\\

\begin{figure}
\center
\begin{tabular}{c}
\includegraphics[width=200px]{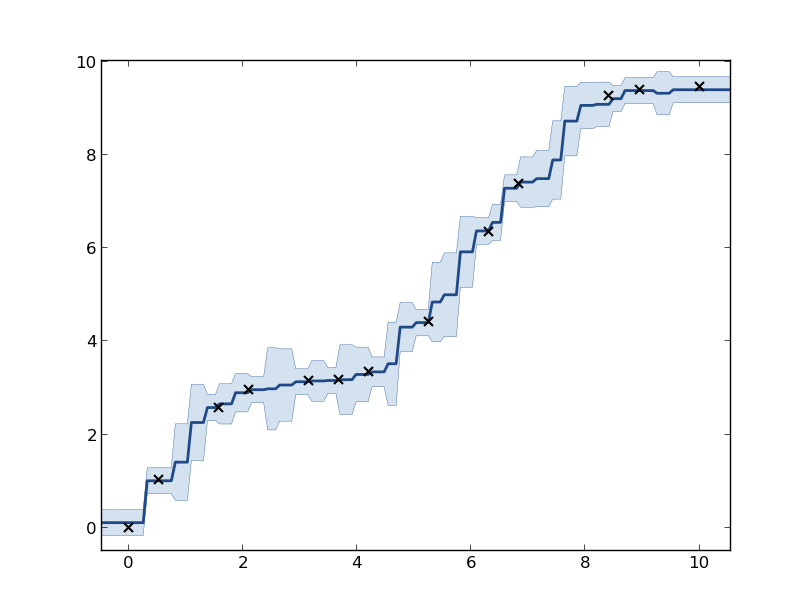}\\
Random Forest\\
\includegraphics[width=200px]{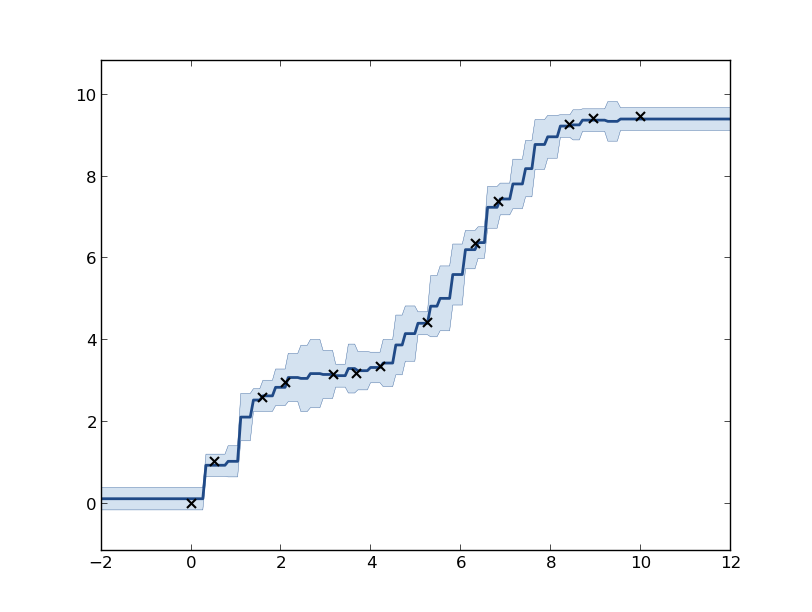}\\
Fast Cluster\\
\end{tabular}
\caption{GP posterior distribution on toy regression dataset}
\label{fig:rf_posterior}
\end{figure}

\begin{figure}[h!]
\center
\begin{tabular}{c}
\includegraphics[width=200px]{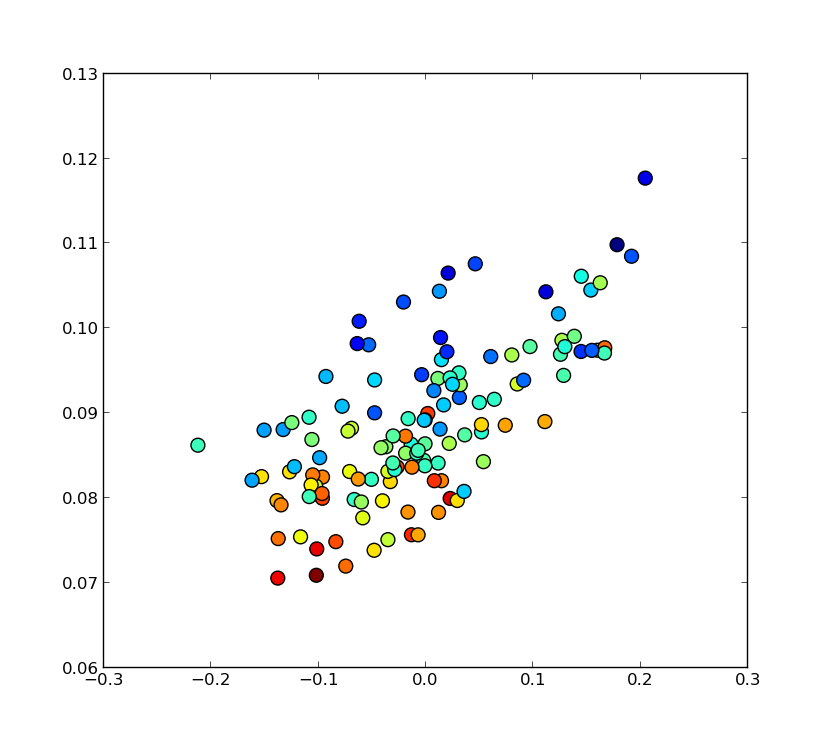}\\
PCA\\
\includegraphics[width=200px]{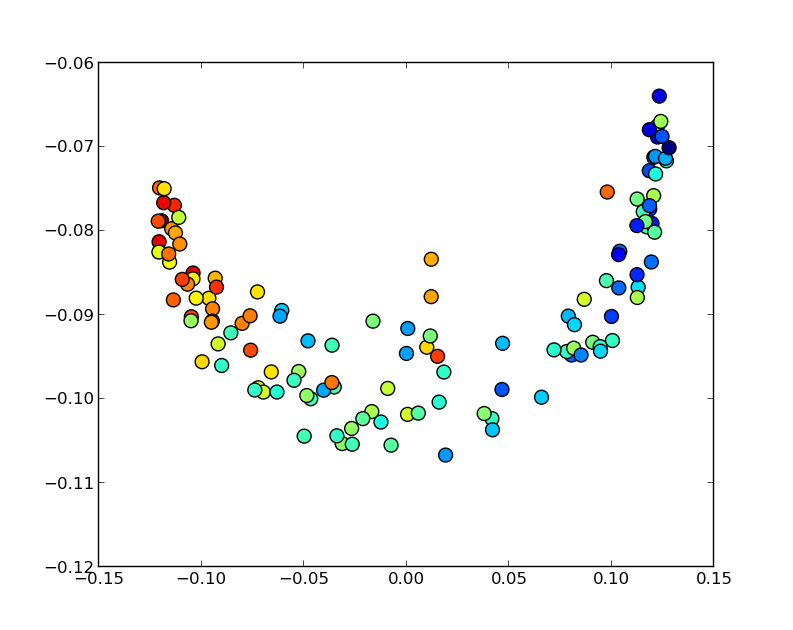}\\
Fast Cluster PCA\\
\includegraphics[width=200px]{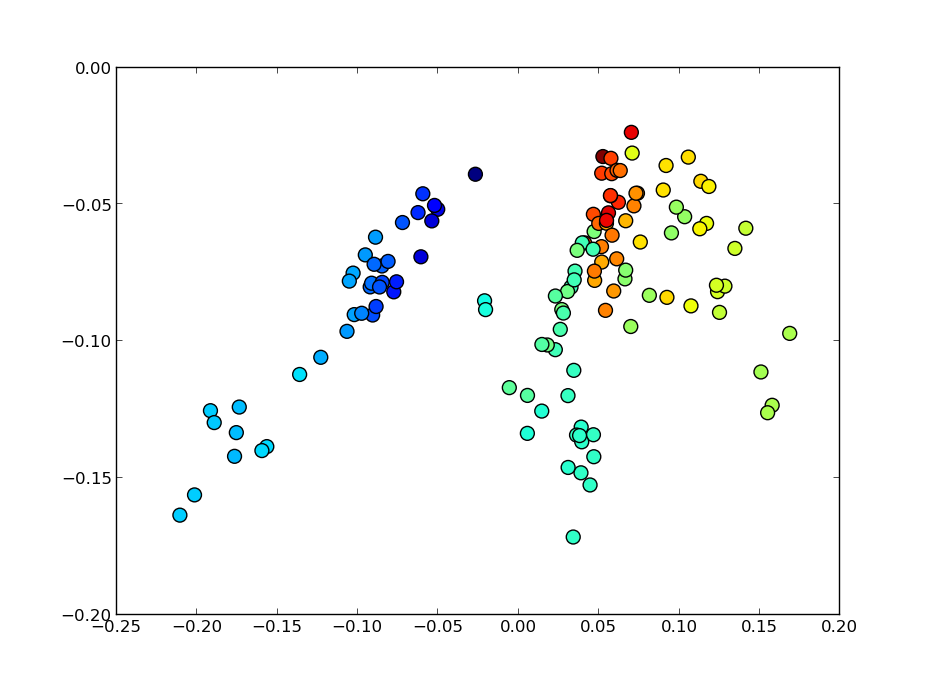}\\
Random Forest PCA\\
\end{tabular}
\caption{PCA vs FC Kernel PCA vs RF Kernel PCA on `bodyfat' dataset}
\label{fig:rf_pca}
\end{figure} 

\subsubsection{Fast Cluster Kernel}

The Fast Cluster Kernel is based around a fast random partition algorithm inspired by properties of Random Forest. The algorithm is as follows: First, sample a subset of the dimensions, then select a random number of cluster centers. Assign every point to the cluster associated with its nearest center measured only on the subsampled dimensions. This is outlined in Algorithm \ref{alg:rc_sc}.\\

This cluster scheme was constructed largely for it's simplicity and speed, scaling as $O(N)$. From Figure \ref{fig:rf_posterior} we can see that once again it results in a distribution over piece-wise constant functions that is non-stationary, with smoother variances than the Random Forest kernel. It is also worth noting that when being used for classification/regression the Fast Cluster kernel can easily be used in a ``semi-supervised'' manner, by training it on a larger set of $\mathbf{X}$ than is used for the final learning task.

\begin{algorithm}[tb]
   \caption{Random Clustering Partition Algorithm}
   \label{alg:rc_sc}
\begin{algorithmic}
   \STATE {\bfseries Input:} $\mathbf{X} \in \mathcal{R}^{N \times D}, h$
   \STATE $\vec{d} \sim {\rm Bernoulli}(.5, D)$
   \STATE $s \sim {\rm DiscreteUniform}(0,h)$
   \STATE $\mathcal{C} \sim {\rm Sample}([1,...,N], 2^s)$
   \FOR{$a \in \mathcal{D}$}
   \STATE $\varrho(a) = argmin_{c \in \mathcal{C}} \left(\|(\mathbf{X}_c - \mathbf{X}_a) \odot \vec{d}\|\right)$
   \ENDFOR
\end{algorithmic}
\end{algorithm}

\section{Experiments}

\subsection{Kernel Efficacy}

We compare the results of using the Random Forest kernel against the Linear kernel and the Radial Basis Function with and without Automatic Relevance Detection using standard GP training methodology\footnote{Each kernel is optimized for training log-likelihood using a gradient-based optimizer with 20 random restarts.} on 6 real-world regression problems from the UCI Machine Learning repository. Both Random Forest and Fast Cluster were trained with $m = 200$.

The evaluation criteria that is most important for a Gaussian Process is the test log-likelihood, as this directly shows how well the posterior distribution has been learnt. It factors in not only how well the test data has been predicted, but also the predicted variances and covariances of the test data. Figure \ref{fig:ll_results} shows that both the Random Forest kernel and Fast Cluster kernel greatly outperform the standard kernels with respect to this measure on a variety of datasets. The graph shows the test log-likelihood on a symmetric log scale as the discrepancy is very large.

As GPs are often used simply as a point estimate regression method, it is still important to assess the kernels performance with respect to MSE. Figure \ref{fig:mse_results} shows that the Random Forest kernel predicts better in the majority of cases, though the improvement is not as pronounced as the improvement in log-likelihood. This shows that while the the resulting joint predictive posterior for the standard kernels can be very poorly fit, they may still result in accurate posterior mean predictions. 

\begin{figure}[h!]
\center
\includegraphics[width=250px]{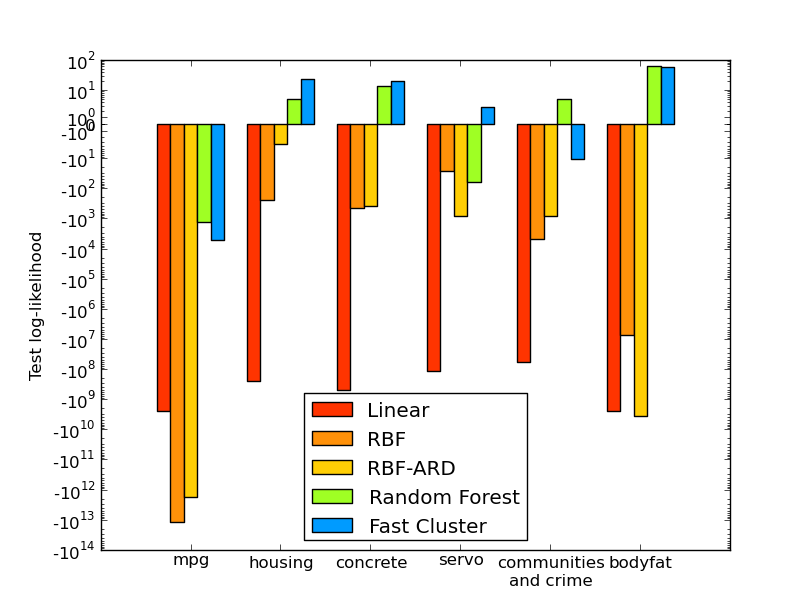}
\caption{Test log-likelihood on real-world regression problems (higher is better)}
\label{fig:ll_results}
\end{figure} 

\begin{figure}[h!]
\center
\includegraphics[width=250px]{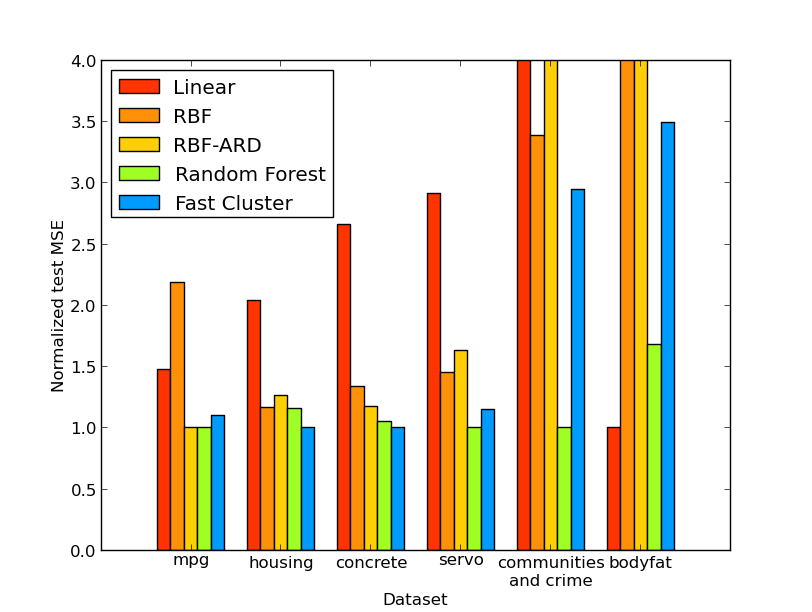}
\caption{Test MSE on real-world regression problems (lower is better)}
\label{fig:mse_results}
\end{figure}

\subsection{Approximation quality}

To understand the effect of the kernel approximation on the resulting predictions, we graph the MSE performance of both kernels for different values of $m$. In Figure \ref{fig:approximation} we can see that for low samples, the Random Forest kernel performs nearly as well as its optimum, whereas the Fast Cluster kernel improves continually untill around $200$ samples. This appears roughly consistent with what we might expect from the theoretical convergence of the kernel, as this equates to an error variance of approximately $.002$, which is a quite small element-wise error for a well-conditioned matrix with elements in the range $[0,1]$.

\begin{figure}[h!]
\center
\includegraphics[width=250px]{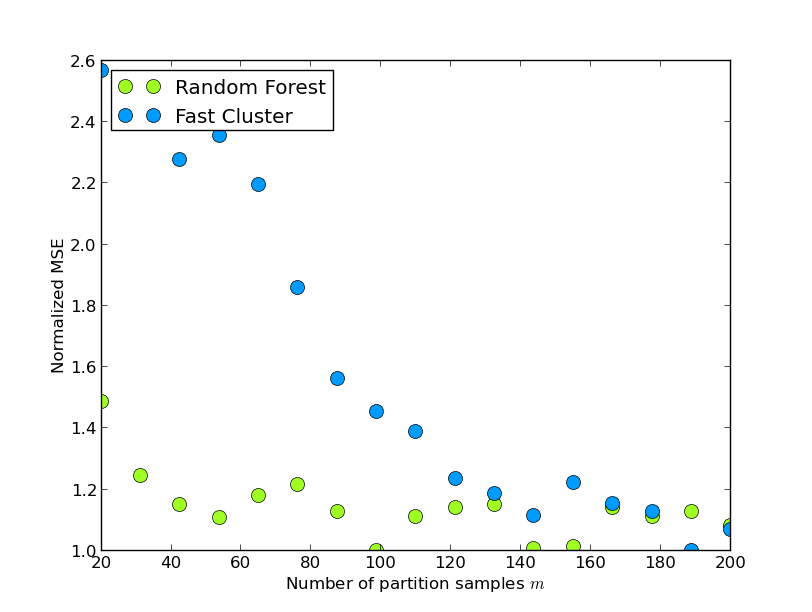}
\caption{Log-likelihood vs $m$ for `mpg' dataset}
\label{fig:approximation}
\end{figure} 

\subsection{Scalability}

To show that the scalability of the iterative algorithms is as we would expect from the computational analysis in Section \ref{sec:efficient_eval}, we show the runtime for Kernel PCA with the different kernels. Both Random Forest and Fast Cluster are trained with 100 random partitions. Note that Squared Exp PCA has a constant advantage, as the code is executed almost entirely in C, whereas Random Forest and Fast Cluster are partially executed in python. The experiments were run on a commidity PC with an Intel Core 2 Duo E8500 CPU and 8GB RAM.\\

Figure \ref{fig:pca_scaling} is a log-log plot, so the exponent of the scaling can be seen in the gradient of the curves. As predicted, the scaling of Fast Cluster is $O(N)$, processing around 100,000 points per minute on this machine. In practice, for these size datasets, Random Forest seems to scale as approximately $O(N^{1.5})$, while the RBF remains faithful to its theoretical $O(N^3)$.

As a comment on future scaling, it is worth noting here that both the Fast Cluster and Random Forest algorithms, as well as the matrix-vector product, are trivially parallelizable.

\begin{figure}[h!]
\center
\includegraphics[width=250px]{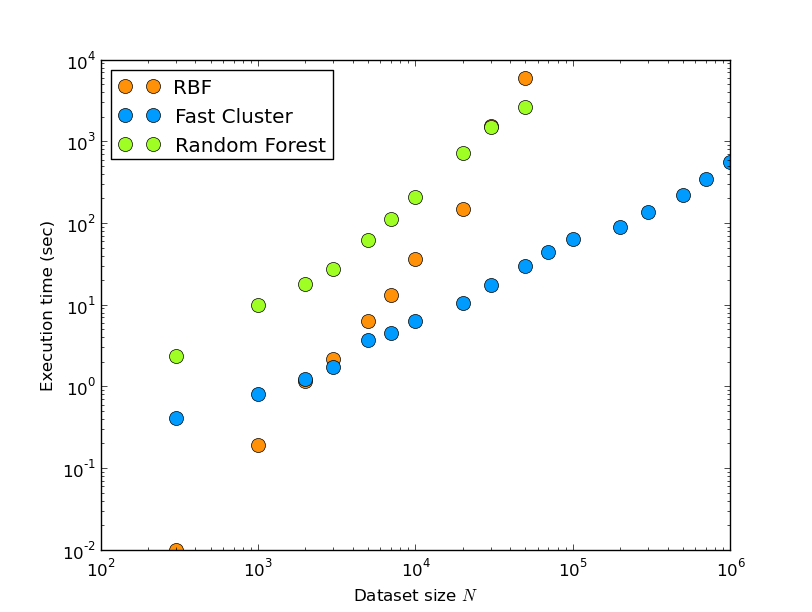}
\caption{Log-log plot of dataset size vs execution time for kernel PCA}
\label{fig:pca_scaling}
\end{figure} 

\section{Conclusion}

We have presented the Random Partition Kernel, a novel method for constructing effective kernels based on a connection between kernels and random partitions. We show how this can be used to simply and intuitively define a large set of kernels, many trivially from existing algorithms. The example kernels constructed using this method, the Random Forest kernel and Fast Cluster kernel, both show excellent performance in modelling real-world regression datasets with Gaussian Processes, substantially outperforming other kenels. We also demonstrate that the kernel allows for linearly scaling inference in GPs, SVMs and kernel PCA.

\bibliography{random_forest_kernel_icml}{}
\bibliographystyle{icmlstylefiles/icml2014}

\end{document}